\def\year{2016}\relax
\documentclass[letterpaper]{article}
\usepackage{aaai16}
\usepackage{times}
\usepackage{helvet}
\usepackage{courier}
\usepackage{amsmath,amsfonts,amsthm,amssymb}
\usepackage{algpseudocode}
\usepackage{algorithm}
\usepackage{footnote}
\usepackage{graphicx,subfigure}
\usepackage{multirow}
\usepackage[title]{appendix}
\newcommand{\ud}{\mathrm d}

\newcommand{\kl}{\mathrm{KL}}

\newcommand{\err}{\mathrm{err}}
\newcommand{\sgn}{\mathrm{sgn}}
\newcommand{\argmin}{\mathrm{argmin}}
\newcommand{\argmax}{\mathrm{argmax}}
\newcommand{\seq}{\mathrm{str}}
\newcommand{\prob}{\mathrm{qs}}

\newtheorem{thm}{Theorem}
\newtheorem{lem}{Lemma}
\newtheorem{prop}{Proposition}
\newtheorem{cor}{Corollary}

\begin{document}
%
\title{Noise-adaptive Margin-based Active Learning and Lower Bounds under Tsybakov Noise Condition}
\author{Yining Wang \and Aarti Singh\\
Machine Learning Department, School of Computer Science, Carnegie Mellon University\\
5000 Forbes Avenue, Pittsburgh PA 15213
}
\maketitle
\begin{abstract}
We present a simple noise-robust margin-based active learning algorithm
to find homogeneous (passing the origin) linear separators
and analyze its error convergence when labels are corrupted by noise.
We show that when the imposed noise satisfies the Tsybakov low noise condition \cite{tnc-ref1,tnc-ref2}
the algorithm is able to adapt to unknown level of noise and achieves optimal statistical rate up to polylogarithmic factors.

We also derive lower bounds for margin based active learning algorithms under Tsybakov noise conditions (TNC)
for the membership query synthesis scenario \cite{mqs88}.
Our result implies lower bounds for the stream based selective sampling scenario \cite{sbss90}
under TNC for some fairly simple data distributions.
Quite surprisingly, we show that the sample complexity cannot be improved
even if the underlying data distribution is as simple as the uniform distribution on the unit ball.
Our proof involves the construction of a well-separated hypothesis set on the $d$-dimensional unit ball
along with carefully designed label distributions for the Tsybakov noise condition.
Our analysis might provide insights for other forms of lower bounds as well.
\end{abstract}

\section{Introduction}

Active learning is an increasingly popular setting in machine learning that makes use of both unlabeled and selectively sampled labeled data \cite{agnostic-active,generalization-active,coarse-sample-complexity}.
In general, an active learning algorithm has access to a large number of unlabeled examples
and has the capacity to request labels of specific examples.
The hope is that by directing label queries to the most informative examples in a feedback-driven way,
we might be able to achieve significant improvements in terms of sample complexity over passive learning algorithms.
For instance, in the problem of learning homogeneous (passing the origin) linear separators,
an exponential improvement in sample complexity could be achieved under the realizable case, 
where the labels are consistent with the optimal linear classifier \cite{colt07-active,active-log-concave}.
For noisy label distributions, a polynomial improvement in sample complexity is more typical \cite{active-minimax,colt07-active,agnostic-active,active-log-concave}.

We consider two active learning scenarios in this paper: the \emph{stream-based selective sampling} scenario \cite{sbss90,generalization-active}, 
under which an algorithm has access to a large number of unlabeled data in a stream and can decide whether to query the label of a specific data point,
and the \emph{membership query synthesis} scenario \cite{mqs88} under which an algorithm has the capacity of synthesizing data points
and obtaining their labels from an oracle.
For the stream-based setting, we analyze a noise-robust margin-based active learning algorithm under the Tsybakov noise condition \cite{tnc-ref1,tnc-ref2}.
We show that the algorithm automatically adjusts to unknown noise levels in the Tsybakov noise condition (TNC)
while achieving the same statistical rate (up to polylogarithmic terms) as non-adaptive algorithms.
This makes margin-based active learning more practical, as the amount of noise in label distributions is usually unknown in practice.

We also study lower bounds for the membership query synthesis setting under Tsybakov noise conditions.
Our lower bound matches previous ones for the stream-based selective sampling setting \cite{active-log-concave,active-survey}.
Quite surprisingly, as a consequence of our lower bound,
we show that stream-based active learning algorithms cannot do better even if the underlying data distribution is as simple as the uniform distribution,
It also means the previous proposed margin-based active learning algorithms \cite{colt07-active,active-log-concave} are optimal under their specific problem settings.
To the best of our knowledge, such results are not implied by any previous lower bounds on active learning, as we discuss in more detail below.

\section{Related work}

A margin-based active learning algorithm for learning homogeneous linear separators was proposed in \cite{colt07-active} with its sample complexity analyzed
under the Tsybakov low noise condition for the uniform distribution on the unit ball.
The algorithm was later extended to log-concave data distributions \cite{active-log-concave}.
Recently \cite{steve-minimax} introduced a disagreement-based active learning algorithm 
that works for arbitrary underlying data distributions.
For all of the above-mentioned algorithms, given data dimension $d$ and query budget $T$,
the excess risk $\epsilon$ is upper bounded by 
\footnote{In the $\widetilde O(\cdot)$ notation we omit dependency on failure probability $\delta$ and polylogarithmic dependency on $d$ and $T$.}
$\widetilde O((d/T)^{1/2\alpha})$,
where $\alpha$ is a parameter characterizing the noise level in TNC (cf. Eq. (\ref{tnc-md}) in Section \ref{sec:setup}).
These algorithms are not noise-adaptive; that is, the selection of key algorithm parameters depend on the noise level $\alpha$,
which may not be available in practice.

In \cite{hanneke-adaptive} a noise-robust disagreement-based algorithm was proposed for agnostic active learning.
The analysis was further improved in \cite{zhang-chaudhuri} by replacing the disagreement coefficient with a provably smaller quantity.
However, their error bounds are slightly worse under our settings, as we discuss in Section \ref{sec:discussion}.
Also, in both analysis the desired accuracy $\epsilon$ is fixed, while in our setting the number of active queries $T$ is fixed.
Under the one-dimensional threshold learning setting, \cite{alt13} proposed a noise-adaptive active learning algorithm
inspired by recent developments of adaptive algorithms for stochastic convex optimization \cite{juditsky-nesterov}.
For multiple dimensions,
it was shown recently in \cite{margin-active-adversarial} that a noise-robust variant of margin-based active learning
achieves near optimal noise tolerance.
The authors analyzed the maximum amount of adversarial noise an algorithm can tolerate under the constraints of 
constant excess risk
and polylogarithmic sample complexity,
which is equivalent to an exponential rate of error convergence.
In contrast, we study the rate at which the excess risk (relative to Bayes optimal classifier) converges to zero with number of samples that are not restricted to be polylogarithmic.  

In terms of negative results, it is well-known that the $\widetilde O((d/T)^{1/2\alpha})$ upper bound is tight up to polylogarithmic factors.
In particular, Theorem 4.3 in \cite{active-survey} shows that
for any stream-based active learning algorithm, 
there exists a distribution $P_{XY}$ satisfying TNC such that
the excess risk $\epsilon$ is lower bounded by $\Omega((d/T)^{1/2\alpha})$.
The marginal data distribution $P_{X}$ is constructed in an adversarial manner
and it is unclear whether the same lower bound applies when $P_{X}$ is some simple (e.g., uniform or Gaussian) distribution.
\cite{active-log-concave} proved lower bounds for stream-based active learning under each log-concave data distribution.
However, their proof only applies to the separable case and shows an exponential error convergence.
In contrast, we consider Tsybakov noise settings with parameter $\alpha\in(0,1)$,
for which polynomial error convergence is expected \cite{active-survey}.

\cite{active-minimax} analyzed the minimax rate of active learning under the membership query synthesis model (cf. Section \ref{sec:mqs}).
Their analysis implies a lower bound for stream-based setting when the data distribution is uniform or bounded from below (cf. Proposition \ref{prop_pool_probe_reduction} and \ref{prop_regression-excess-TNC}).
However, their analysis focuses on the nonparametric setting where the Bayes classifier $f^*$ is not assumed to have a parametric form such as linear.
Consequently, their is a polynomial gap between their lower bound and the upper bound for linear classifiers.

\section{Problem setup and notations}\label{sec:setup}

We assume the data points $(x, y)\in\mathcal X\times\mathcal Y$ are drawn 
from an unknown joint distribution $P_{XY}$,
where $\mathcal X$ is the instance space and $\mathcal Y$ is the label space.
Furthermore, $x$ are drawn in an i.i.d. manner.
In this paper we assume that $\mathcal X=\mathcal S^d\subseteq \mathbb R^d$ is the unit ball in $\mathbb R^d$
and $\mathcal Y=\{+1, -1\}$.

The goal of active learning is to find a classifier $f:\mathcal X\to \mathcal Y$ such that the generalization error
$\err(f) = \mathbb E_{(x,y)\sim P}[\ell(f(x),y)]$ is minimized.
Here $\ell(f(x),y)$ is a loss function between the prediction $f(x)$ and the label $y$.
Under the binary classification setting with $\mathcal Y=\{+1,-1\}$, the $0/1$ classification loss is of interest, where 
$\ell(f(x), y) = I[yf(x) > 0]$ with $I[\cdot]$ the indicator function.
In this paper we consider the case where the Bayes classifier $f^*$ is linear,
that is, $f^*(x) = \argmax_{y\in\{+1,-1\}}{\Pr(Y=y|X=x)} = \sgn(w^*\cdot x)$ with $w^*\in\mathbb R^d$.
Note that the Bayes classifier $f^*$ minimizes the generalization 0/1 error $\Pr(Y\neq f^*(X))$.
Given the optimal classifier $f^*$, we define the \emph{excess risk} of a classifier $f$ under $0/1$ loss as $\err(f)-\err(f^*)$.
Without loss of generality, we assume all linear classifiers $f(x)=\sgn(w\cdot x)$ have norm $\|w\|_2=1$.
We also use $B_{\theta}(w, \beta)$ to denote the model class $\{f(x)=w'\cdot x|\theta(w', w)\leq\beta, \|w'\|_2=1\}$
consisting of all linear classifiers that are close to $w$ with an angle at most $\beta$.
Here $\theta(w',w)=\mathrm{arccos}(w'\cdot w)$ is the angle between $w'$ and $w$.
We use $\log$ to denote $\log_2$ and $\ln$ to denote the natural logarithm.

\paragraph{Tsybakov noise condition}
For the conditional label distribution $P_{Y|X}$,
we consider a noise model characterized by the Tsybakov low noise condition (TNC) along the optimal hyperplane.
Various forms of the TNC condition for the one-dimensional and multi-dimensional active learning 
are explored in \cite{tnc-1d,alt13,colt07-active,active-log-concave}
and have been increasingly popular in the active learning literature.
In this paper, we use the following version of the TNC condition:
there exists a constant $0<\mu<\infty$ such that
for all linear classifiers
\footnote{To simplify notations, we will interchangeably call $w$, $f$ and $\sgn(f)$ as linear classifiers.}
$w\in\mathbb R^d$, $\|w\|_2=1$ the following lower bound on excess risk holds:
\begin{equation}
\mu\cdot \theta(w,w^*)^{1/(1-\alpha)} \leq \err(w)-\err(w^*),
\label{tnc-md}
\end{equation}
with $\alpha\in [0,1)$ a parameter characterizing the noise level in the underlying label distribution.

\paragraph{Stream-based selective sampling}\label{sec:sbss}
The stream-based selective sampling scheme was proposed in \cite{sbss90,generalization-active}.
Under the stream-based setting an algorithm has access to a stream of unlabeled data points
and can request labels of selected data points in a feedback-driven manner.
Formally speaking, a stream-based active learning algorithm operates in iterations and for iteration $t$ it does the following:
\begin{enumerate}
\item The algorithm obtains an unlabeled data point $x_t$, sampled from the marginal distribution $P_{X}$.
\item The algorithm then decides, based on previous labeled and unlabeled examples, whether to accept $x_t$ and request its label.
If a request is made, it obtains label $y_t$ sampled from the conditional distribution $p(\cdot|x_t)$.
\end{enumerate}
Finally, after a finite number of iterations the algorithm outputs a hypothesis $\widehat f(x)=\sgn(\widehat w\cdot x)$.
We use $\mathcal A_{d,T}^{\seq}$ to denote all stream-based selective sampling algorithms that operate on $\mathcal X=\mathcal S^d$
and make no more than $T$ label requests.

{
The stream-based selective sampling setting is slightly weaker than the pool based active learning setting considered in \cite{colt07-active,active-log-concave}.
For pool-based active learning, an algorithm has access to the entire pool $(x_1,x_2,\cdots)$ of unlabeled data before it makes any query requests.
We remark that all margin-based active learning algorithms proposed in \cite{colt07-active,active-log-concave,margin-active-adversarial}
actually work under the stream-based setting.
}

\paragraph{Membership query synthesis}\label{sec:mqs}
An alternative active learning scenario is the synthetic query setting
under which an active learning algorithm is allowed to synthesize queries and ask an oracle to label them.
The setting is introduced in \cite{mqs88} and considered in \cite{tnc-1d,active-minimax,alt13}.
Formally speaking, a query synthesis active learning algorithm operates in iterations and for iteration $t$ it does the following:
\begin{enumerate}
\item The algorithm picks an arbitrary data point $x_t\in\mathcal X$, based on previous obtained labeled data.
\item The algorithm is returned with label $y_t$ sampled from the conditional distribution $p(\cdot|x_t)$.
\end{enumerate}
Finally, after $T$ iterations the algorithm outputs a hypothesis $\widehat f(x)=\sgn(\widehat w\cdot x)$,
where $T$ is the total number of label queries made.
We use $\mathcal A_{d,T}^{\prob}$ to denote all membership query algorithms that operate on $\mathcal X=\mathcal S^d$ and make no more than $T$ label queries.

{
We remark that the synthetic query setting is more powerful than stream-based selective sampling.
More specifically, we have the following proposition.
It can be proved by simple reductions and the proof is deferred to Appendix \ref{appsec:technical_prop} in \cite{mdactive-arxiv}.
\begin{prop}
Fix $d,T$.
For any marginal distribution $P_{X}$ and family of conditional label distributions $\mathcal P$ the following holds:
\begin{equation}
\inf_{A\in\mathcal A_{d,T}^{\prob}}{\sup_{P_{Y|X}\in\mathcal P}{\mathbb E[L(\widehat w,w^*)]}}
\leq \inf_{A\in\mathcal A_{d,T}^{\seq}}{\sup_{P_{Y|X}\in\mathcal P}{\mathbb E[L(\widehat w,w^*)]}},
\label{eq_pool_probe_reduction}
\end{equation}
where $L(\widehat w,w^*)=\err(\widehat w)-\err(w^*)$ is the excess risk of output hypothesis $\widehat w$.
\label{prop_pool_probe_reduction}
\end{prop}
}

\section{Noise-adaptive upper bounds}\label{sec:upperbound}

In this section we prove the following main theorem, which provides an upper excess-risk bound on stream-based active learning algorithms
that adapt to different noise levels under the TNC condition.
\begin{thm}
Fix $\delta\in(0,1), r\in (0,1/2), d\geq 4$ and $T\geq 4$.
Suppose $P_{X}$ is the uniform distribution on the unit ball $\mathcal S^d$.
There exists a stream-based active learning algorithm $A\in\mathcal A_{d,T}^\seq$ such that
for any label distribution $P_{Y|X}$ that satisfies Eq. (\ref{tnc-md})
with parameters $\mu>0$ and $1/(1+\log(1/r))\leq\alpha < 1$, the following holds with probability $\geq 1-\delta$:
\begin{equation}
\err(\widehat w)-\err(w^*) = \widetilde O\left(\left(\frac{d+\log(1/\delta)}{T}\right)^{1/2\alpha}\right).
\label{eq_main}
\end{equation}
Here $\widehat w$ is the output decision hyperplane of $A$, $w^*$ is the Bayes classifier and in $\widetilde O(\cdot)$ we omit dependency on $r,\mu$ and 
polylogarithmic dependency on $T$ and $d$.
\label{thm_main}
\end{thm}

Theorem \ref{thm_main} shows one can achieve the same error rate (up to polylogarithmic factors) as previously proposed
algorithms \cite{colt07-active,active-log-concave} without knowing noise level in the label distribution (characterized by $\mu$ and $\alpha$).
To prove Theorem \ref{thm_main}, we explicitly construct an algorithm that is adaptive to unknown noise levels (Algorithm \ref{alg_main}).
The algorithm is in principle similar to the margin-based active learning algorithms proposed in \cite{colt07-active,active-log-concave},
with the setting of margin thresholds a slight generalization of \cite{margin-active-adversarial}.
However, we analyze it under the noise-adaptive TNC setting, which has not been considered before specifically for margin-based active learning algorithms.

In the remainder of this section we describe the noise-adaptive algorithm we analyzed and provide a proof sketch for Theorem \ref{thm_main}.
Our analysis can be easily generalized to log-concave densities, with details in Appendix \ref{appsec:logconcave} in \cite{mdactive-arxiv}.

\begin{algorithm}[t]
\caption{A noise-adaptive margin-based active learning algorithm}
\begin{algorithmic}[1]
	\State \textbf{Parameters}: data dimension $d$, query budget $T$, failure probability $\delta$, shrinkage rate $r$.
	\State \textbf{Initialize}: $E = \frac1{2}\log T$,
		$n=T/E$, $\beta_0=\pi$, random $\widehat w_0$ with $\|\widehat w_0\|=1$. 
	\For{$k=1$ to $E$}
		\State $W=\emptyset$.  Set $b_{k-1} = \frac{2\beta_{k-1}}{\sqrt{d}}\sqrt{E(1+\log(1/r))}$ if $k>1$ and $b_{k-1} = +\infty$ if $k=1$.
		\While{$|W| < n$}
			\State Obtain a sample $x$ from $P_{X}$.
			\State If $|\widehat w_{k-1}\cdot x| > b_{k-1}$, reject; 
				otherwise, ask for the label of $x$, and put $(x,y)$ into $W$.
		\EndWhile
		\State Find $\widehat w_k\in B_{\theta}(\widehat w_{k-1}, \beta_{k-1})$ that minimizes the empirical 0/1 error
		  $\sum_{(x,y)\in W}{I[yw\cdot x<0]}$.
		\State Update: $\beta_k\gets r\cdot\beta_{k-1}$, $k\gets k+1$.
	\EndFor
	\State \textbf{Output}: the final estimate $\widehat w_E$.
\end{algorithmic}
\label{alg_main}
\end{algorithm}

\subsection{The algorithm}

We present Algorithm \ref{alg_main}, a margin-based active learning algorithm that adapts to unknown $\alpha$ and $\mu$ values
in the TNC condition in Eq.~(\ref{tnc-md}).
Algorithm \ref{alg_main} admits 4 parameters: $d$ is the dimension of the instance space $\mathcal X$;
$T$ is the sampling budget (i.e., maximum number of label requests allowed);
$\delta$ is a confidence parameter;
$r\in (0,1/2)$ is the shrinkage rate of the hypothesis space for every iteration in the algorithm;
smaller $r$ allows us to adapt to smaller $\alpha$ values
but will result in a larger constant in the excess risk bound.
The basic idea of the algorithm is to split $T$ label requests into $E$ iterations,
using the optimal passive learning procedure within each iteration
and reducing the scope of search for the best classifier after each iteration.

The key difference between the adaptive algorithm and the one presented in \cite{colt07-active}
is that in Algorithm~\ref{alg_main} the number of iterations $E$ 
as well as other parameters (e.g., $b_k,\beta_k$)
are either not needed or do not depend on the noise level $\alpha$,
and the number of label queries is divided evenly across the iterations.
Another difference is that in our algorithm the sample budget $T$ is fixed
while in previous work the error rate $\epsilon$ is known.
It remains an open problem whether there exists a tuning-free active learning algorithm
when a target error rate $\epsilon$ instead of query budget $T$ is given \cite{alt13}.

\subsection{Proof sketch of Theorem \ref{thm_main}}\label{sec:sketch_upperbound}

In this section we sketch the proof of Theorem \ref{thm_main}.
The complete proof is deferred to Appendix \ref{appsec:proof_upperbound} in \cite{mdactive-arxiv}.

We start by defining some notations used in the proof.
Let $\mathcal F_k=B_{\theta}(\widehat w_{k-1},\beta_{k-1})$ be the hypothesis space considered in the $k$th iteration of Algorithm \ref{alg_main}.
Let $D_k$ be the obtained labeled examples and $S_1^{(k)}=\{x||\widehat w_{k-1}\cdot x|\leq b_{k-1}\}$ be the acceptance region at the $k$th iteration.
By definition, $D_k\subseteq S_1^{(k)}$.
Let $w_k^*=\argmax_{w\in\mathcal F_k}{\err(w|S_1^{(k)})}$ be the optimal classifier in $\mathcal F_k$ with respect to the generalization 0/1 loss in the acceptance region $S_1^{(k)}$.
Using similar techniques as in \cite{colt07-active}, 
it can be shown that 
with probability $\geq 1-\delta$ the following holds:
\begin{equation}
\err(\widehat w_k)-\err(w_k^*) \leq \beta_{k-1}\epsilon,
\label{eq_opt_passive_maintext}
\end{equation}
where $\epsilon$ is of the order $\widetilde O(\sqrt{\frac{d+\log(1/\delta)}{T}})$.

Eq. (\ref{eq_opt_passive_maintext}) shows that if $\beta_{k-1}$ is small then we get good excess risk bound.
However, $\beta_{k-1}$ should be large enough so that $\mathcal F_k$ contains the Bayes classifier $w^*$ (i.e., $w_k^*=w^*$).
In previous analysis \cite{colt07-active,active-log-concave} the algorithm parameters $\beta_{k-1}$ and $b_{k-1}$
are carefully selected using the knowledge of $\alpha$ and $\mu$ so that $w_k^*=w^*$ for all iterations.
This is no longer possible under our setting because the noise parameters $\alpha$ and $\mu$ are unknown.
Instead, we show that there exists a ``tipping point" $k^*\in\{1,2,\cdots,E-1\}$ depending and $\alpha$ and $\mu$ that divides Algorithm \ref{alg_main} into two phases:
in the first phase ($k\leq k^*$) everything behaves the same with previous analysis for non-adaptive margin-based algorithm;
that is, we have per-iteration excess error upper bounded by Eq. (\ref{eq_opt_passive_maintext}) and the optimal Bayes classifier $w^*$ is contained in the constrained hypothesis space
$\mathcal F_k$ (i.e., $w_k^*=w^*$) for all $k\leq k^*$.
Formally speaking, we have the following two lemmas which are proved in Appendix \ref{appsec:proof_upperbound} in \cite{mdactive-arxiv}.
\begin{lem}
Suppose $r\in (0,1/2)$ and $1/(1+\log(1/r))\leq \alpha < 1$.
With probability at least $1-\delta$,
\begin{equation}
\err(\widehat w_{k^*})-\err(w_{k^*}^*) 
\leq  \beta_{k^*-1}\epsilon 
\leq \frac{\epsilon^{1/\alpha}}{r^{\frac{1+\alpha}{\alpha}}\mu^{\frac{1-\alpha}{\alpha}}}.
\label{eq_lem1}
\end{equation}
\label{lem1}
\end{lem}
\begin{lem}
With probability $\geq 1-\delta E$, 
$w_k^* = w^*$ for all $k\leq k^*$.
\label{lem2}
\end{lem}

After iteration $k^*$, the optimal Bayes classifier $w^*$ diverges from $w_k^*$ and we can no longer apply Eq. (\ref{eq_opt_passive_maintext}) directly
to bound the excess risk between $\widehat w_k$ and $w^*$.
However, for $k>k^*$ the constrained hypothesis space $\mathcal F_k$ is quite small and the empirical estimator $\widehat w_k$ cannot deviate much from $\widehat w_{k-1}$.
In particular, we have the following lemma, which is proved in Appendix \ref{appsec:proof_upperbound} in \cite{mdactive-arxiv}.
\begin{lem}
Suppose $r\in(0,1/2)$. With probability at least $1-\delta E$, we have
\begin{equation}
\err(\widehat w_E)-\err(\widehat w_{k^*}) \leq \frac{r}{1-r} \beta_{k^*-1}\epsilon. 
\end{equation}
\label{lem3}
\end{lem}

Combining Lemma \ref{lem1},\ref{lem2} and \ref{lem3} we can upper bound the excess risk $\err(\widehat w_E)-\err(w^*)$
by $\widetilde O(\epsilon^{1/\alpha})$, which corresponds to $\widetilde O((\frac{d+\log(1/\delta)}{T})^{1/2\alpha})$ in Eq. (\ref{eq_main}).
The complete proof is deferred to Appendix \ref{appsec:proof_upperbound} in \cite{mdactive-arxiv}.

\subsection{Extension to log-concave densities}\label{sec:logconcave}

Following recent developments in margin-based active learning \cite{colt07-active,active-log-concave},
Theorem \ref{thm_main} can be further generalized to the case when the data distribution $P_X$ has \emph{log-concave densities},
which includes the uniform data distribution.
A density function $g$ is said to be \emph{log-concave} if $\log g(\cdot)$ is a concave function.
Many popular distributions have log-concave densities, including Gaussian distribution and uniform distribution.
We say the data distribution $P_{X}$ is \emph{isotropic}
if the mean of $P_X$ is zero and the covariance matrix of $P_{X}$ is the identity.
Theorem \ref{thm_main_logconcave} shows that, with slight modifications, Algorithm \ref{alg_main} can be generalized
to the case when the data distribution $P_{X}$ is log-concave and isotropic.
Its proof is similar to the one in \cite{active-log-concave} and is deferred to Appendix \ref{appsec:logconcave} in \cite{mdactive-arxiv}.
\begin{thm}
Fix $\delta\in(0,1), r\in (0,1/2), d\geq 4$ and $T\geq 4$.
Suppose $P_{X}$ is an isotropic log-concave distribution on the unit ball $\mathcal S^d$
and $P_{Y|X}$ satisfies Eq. (\ref{tnc-md}) with parameters $\mu>0$ and $1/(1+\log(1/r))\leq\alpha < 1$
Let $\widehat w$ be the output of Algorithm \ref{alg_main} run with $b_{k-1}=C_1\beta_{k-1}\log T$ and the other parameters unchanged.
\footnote{$C_1$ is an absolute constant. See Lemma \ref{lem_marginerror_logconcave}, \ref{lem_opt_passive_logconcave} in Appendix \ref{appsec:logconcave}  
and Theorem 8 in \cite{active-log-concave} for details.}
Then with probability at least $1-\delta$ the following holds:
\begin{equation}
\err(\widehat w)-\err(w^*) = \widetilde O\left(\left(\frac{d+\log(1/\delta)}{T}\right)^{1/2\alpha}\right).
\label{eq_main}
\end{equation}
Here $\widehat w$ is the output decision hyperplane of $A$, $w^*$ is the Bayes classifier and in $\widetilde O(\cdot)$ we omit dependency on $r,\mu$ and 
polylogarithmic dependency on $T$ and $d$.
\label{thm_main_logconcave}
\end{thm}

\section{Lower bounds}

We prove lower bounds for active learning under the membership query synthesis setting.
Since the query synthetic setting is more powerful than the stream-based setting as shown in Proposition \ref{prop_pool_probe_reduction},
our result implies a lower bound for stream-based selective sampling.
Our lower bound for membership query synthesis setting is for a slightly different version of TNC, 
which implies TNC in Eq. (\ref{tnc-md}) for distributions that are bounded from below (including the uniform distribution).
This shows that both Algorithm \ref{alg_main} and previous margin-based algorithms \cite{colt07-active}
achieve the minimax rate (up to polylogarithmic factors) under the uniform distribution on the unit ball.

To facilitate our analysis for the query synthesis setting, in this section we adopt a new formulation of TNC condition
in terms of the label distribution function $\eta(x)=\Pr(y=1|x)$.
Formally speaking, we assume that there exist constants $0<\mu_0<\infty$ and $\alpha\in[0,1)$ such that for all $x\in X$
the following holds:
\begin{equation}
\mu_0\cdot |\varphi(x,w^*)|^{\alpha/(1-\alpha)} \leq \big|\eta(x)-1/2\big|,
\label{tnc-angle}
\end{equation}
where $w^*$ is the Bayes classifier with respect to $\eta(\cdot)$ and $\varphi(x,w^*):=\frac{\pi}{2}-\theta(x,w^*)\in[-\frac{\pi}{2},\frac{\pi}{2}]$
is the signed acute angle between $x$ and the decision hyperplane associated with $w^*$.
Similar formulation was also used in \cite{tnc-1d,active-minimax,alt13} to analyze active learning algorithms under the query synthesis setting.
We also remark that Eq. (\ref{tnc-angle}) implies the excess-risk based TNC condition in Eq. (\ref{tnc-md})
for data distributions with densities bounded from below, as shown in Proposition \ref{prop_regression-excess-TNC}.
Its proof is deferred to Appendix \ref{appsec:technical_prop} in \cite{mdactive-arxiv}.
\begin{prop}
Suppose the density function $g$ associated with the marginal data distribution $P_{X}$ is bounded from below.
That is, there exists a constant $\gamma\in(0,1)$ such that $g\geq \gamma g_0$,
where $g_0\equiv \pi^{-d/2}\Gamma(1+d/2)$ is the uniform distribution on the unit $d$-dimensional ball.
Then Eq. (\ref{tnc-angle}) implies Eq. (\ref{tnc-md}) with $\mu =2(1-\alpha)\mu_0\gamma$.
\label{prop_regression-excess-TNC}
\end{prop}

We now present the main theorem of this section, which establishes a lower bound on the angle between the output classifier $\widehat w$
and the Bayes classifier $w^*$ for the membership query synthesis setting,
assuming the label distribution $P_{Y|X}$ satisfies TNC condition in Eq. (\ref{tnc-angle}).
\begin{thm}
Fix $d\geq 2$, $T$, $\mu_0>0$ and $\alpha\in(0,1)$.
Suppose $\mathcal X=\mathcal S^d$ and $\mathcal Y=\{+1,-1\}$.
Let $\mathcal P_{\alpha,\mu_0}$ denote the class of all conditional label distributions that satisfy the label distribution based TNC condition in Eq. (\ref{tnc-angle})
with parameters $\alpha$, $\mu_0$.
Then the following excess risk lower bound holds: 
\begin{equation}
\inf_{A\in\mathcal A_{d,T}^{\prob}}\sup_{P_{Y|X}\in\mathcal P_{\alpha,\mu_0}}{\mathbb E[\theta(\widehat w,w^*)]} = \Omega\left(\left(\frac{d}{T}\right)^{(1-\alpha)/2\alpha}\right).
\label{eq_lower_bound}
\end{equation}
Here in the $\Omega(\cdot)$ notation we omit dependency on $\mu_0$.
\label{thm_lower_bound}
\end{thm}

{
Theorem \ref{thm_lower_bound} implies a lower bound for excess-risk based TNC in Eq. (\ref{tnc-md}) when the data distribution $P_{X}$
is uniform or bounded from below, as shown in Corollary \ref{cor_lower_bound}.
By Proposition \ref{prop_pool_probe_reduction}, Eq. (\ref{eq_cor_lower_bound}) holds also for stream-based algorithms $\mathcal A_{d,T}^{\seq}$.
We prove Corollary in Appendix \ref{appsec:proof_lowerbound} in \cite{mdactive-arxiv}.
}
\begin{cor}
Fix $d\geq 2,T,\mu,\gamma>0$ and $\alpha\in(0,1)$.
Suppose $\mathcal X=\mathcal S^d$, $\mathcal Y=\{+1-1\}$ and the density of $P_{X}$ is bounded from below with constant $\gamma$.
Let $\mathcal P_{\alpha,\mu}$ denotes the class of all label distributions that satisfy the excess-risk based TNC condition in Eq. (\ref{tnc-md})
with parameters $\alpha,\mu$.
Then the following lower bound holds:
\begin{equation}
\inf_{A\in\mathcal A_{d,T}^{\prob}}\sup_{P_{Y|X}\in\mathcal P_{\alpha,\mu}}{\mathbb E[\err(\widehat w)-\err(w^*)]} = \Omega\left(\left(\frac{d}{T}\right)^{1/2\alpha}\right).
\label{eq_cor_lower_bound}
\end{equation}
Here in the $\Omega(\cdot)$ notation we omit dependency on $\mu$ and $\gamma$.
\label{cor_lower_bound}
\end{cor}

\subsection{Proof sketch of Theorem \ref{thm_lower_bound}}\label{sec:sketch_lowerbound}

In this section we sketch a proof for Theorem \ref{thm_lower_bound}.
The complete proof is deferred to Appendix \ref{appsec:proof_lowerbound} due to space constraints in \cite{mdactive-arxiv}.
{
We assume the data dimension $d\geq 2$ is even.
This does not lose any generality because the lower bounds in Eq. (\ref{eq_lower_bound}) and (\ref{eq_cor_lower_bound})
remain asymptotically the same if $d$ is replaced with $(d+1)$.
}

The main idea of the proof is the construction of a hypothesis set $\mathcal W=\{w_1^*,\cdots,w_m^*\}\subseteq \mathbb R^d$
with $\log|\mathcal W|=\Omega(d)$ such that for any hypothesis pair $(w_i^*,w_j^*)$ the angle $\theta(w_i^*,w_j^*)$ is large
while $\kl(P_{i,T}\|P_{j,T})$ is small
\footnote{For two continuous distributions $P$ and $Q$ with densities $p$ and $q$, their Kullback-Leibler (KL) divergence $\kl(P\|Q)$ is defined as $\int{p(x)\log{\frac{p(x)}{q(x)}}\ud x}$
if $P\ll Q$ and $+\infty$ otherwise.}
.
Here $P_{i,T}$ denotes the distribution of $T$ labels under the label distribution associated with $w_i^*$ (rigorous mathematical definition
of $P_{i,T}$ is given in the appendix of \cite{mdactive-arxiv}).
Intuitively, we want $w_i^*$ and $w_j^*$ to be well separated in terms of the loss function (i.e., $\theta(w_i^*,w_j^*)$)
while being hard to distinguish by any active learning algorithm under a fixed query budget $T$ (implied by the KL divergence condition).

The following lemma accomplishes the first objective by lower bounding $\theta(w_i^*,w_j^*)$.
Its proof is based on the construction of constant-weight codings \cite{const-weight-coding} and is deferred to Appendix \ref{appsec:proof_lowerbound} in \cite{mdactive-arxiv}.
\begin{lem}
Assume $d$ is even.
Fix a parameter $t\in(0,1/4)$. 
There exists a hypothesis set $\mathcal W=\{w_1^*,\cdots,w_m^*\}\subseteq\mathbb R^d$
such that
\begin{equation}
t\leq \theta(w_i^*,w_j^*)\leq 6.5t,\quad\forall i\neq j;
\label{eq_angle_lowerbound}
\end{equation}
furthermore, $\log|\mathcal W|\geq 0.0625d$ for $d\geq 2$.
\label{lem_angle_lowerbound}
\end{lem}

We next tackle the second objective of upper bounding $\kl(P_{i,T}\|P_{j,T})$.
This requires designing label distributions $\{P_{Y|X}^{(i)}\}_{i=1}^m$ such that they satisfy the TNC condition in Eq. (\ref{tnc-angle})
while having small KL divergence between $P_{Y|X}^{(i)}$ and $P_{Y|X}^{(j)}$ for all distinct pairs $(i,j)$.
We construct the label distribution for the $i$th hypothesis as below:
\begin{equation}
P_{Y|X}^{(i)}(Y=1|X=x) = \left\{\begin{array}{l}
\frac{1}{2} + \sgn(w_i^*\cdot x)\cdot \wp(|\varphi(w_i^*, x)|),\\
 \text{if }|\varphi(w_1^*, x)|\leq 6.5t;\\
\frac{1}{2} + \sgn(w_1^*\cdot x)\cdot \wp(|\varphi(w_1^*, x)|),\\
 \text{if }|\varphi(w_1^*,x)| > 6.5t;
\end{array}\right.
\label{eq_pyx}
\end{equation}
where $\varphi(w,x)=\frac{\pi}{2}-\theta(x,w)\in[-\frac{\pi}{2},\frac{\pi}{2}]$ and $\wp$ is defined as
\begin{equation}
\wp(\vartheta) := \min\{2^{\alpha/(1-\alpha)}\mu_0\cdot \vartheta^{\alpha/(1-\alpha)}, 1/2\}.
\label{eq_f}
\end{equation}
A graphical illustration of $P_{Y|X}^{(1)}$ and $P_{Y|X}^{(i)}$ constructed in Eq. (\ref{eq_pyx}) is depicted in Figure \ref{fig_pyx}.
\begin{figure}[t]
\centering
\subfigure[Illustration of $P_{Y|X}^{(1)}$]{\label{fig_p1}\includegraphics[width=0.9\linewidth]{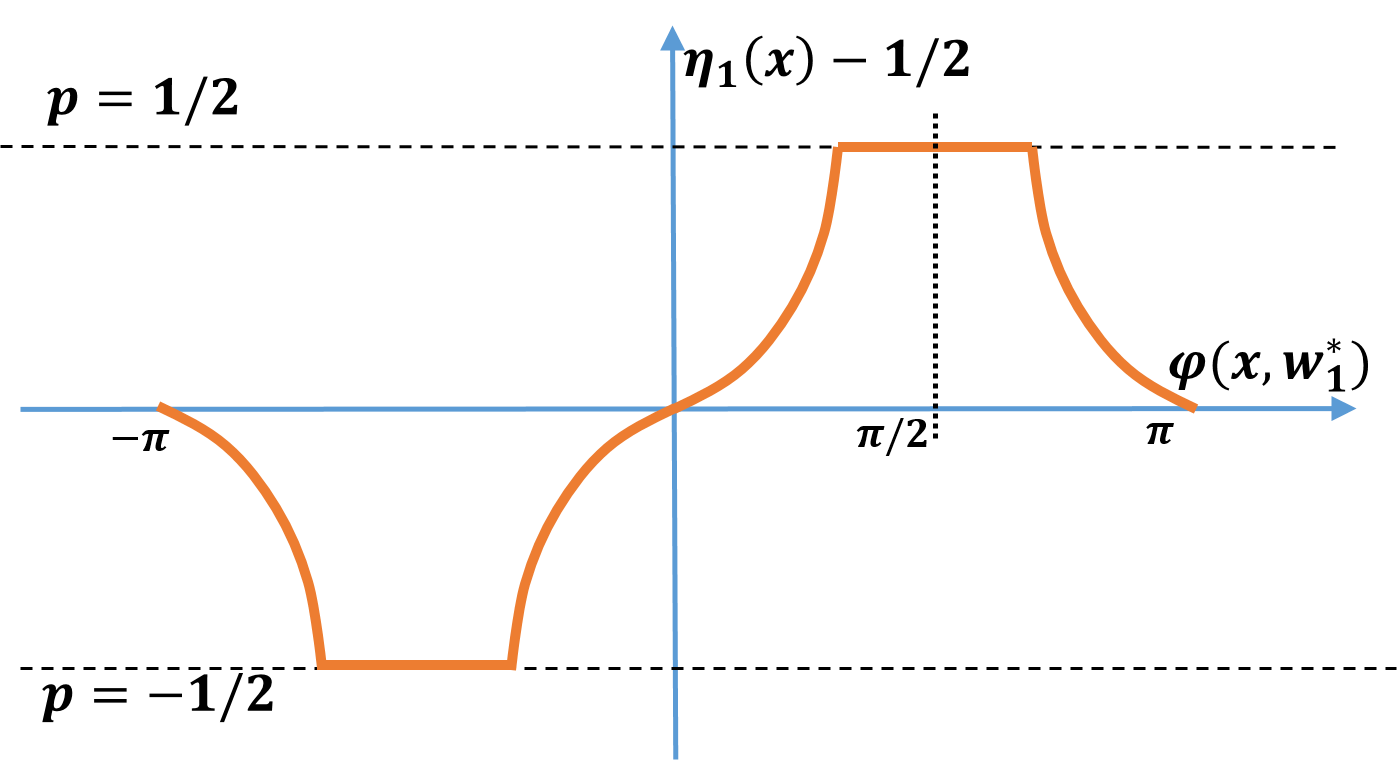}}
\subfigure[Illustration of $P_{Y|X}^{(i)}$, $i\neq 1$]{\label{fig_pi}\includegraphics[width=0.9\linewidth]{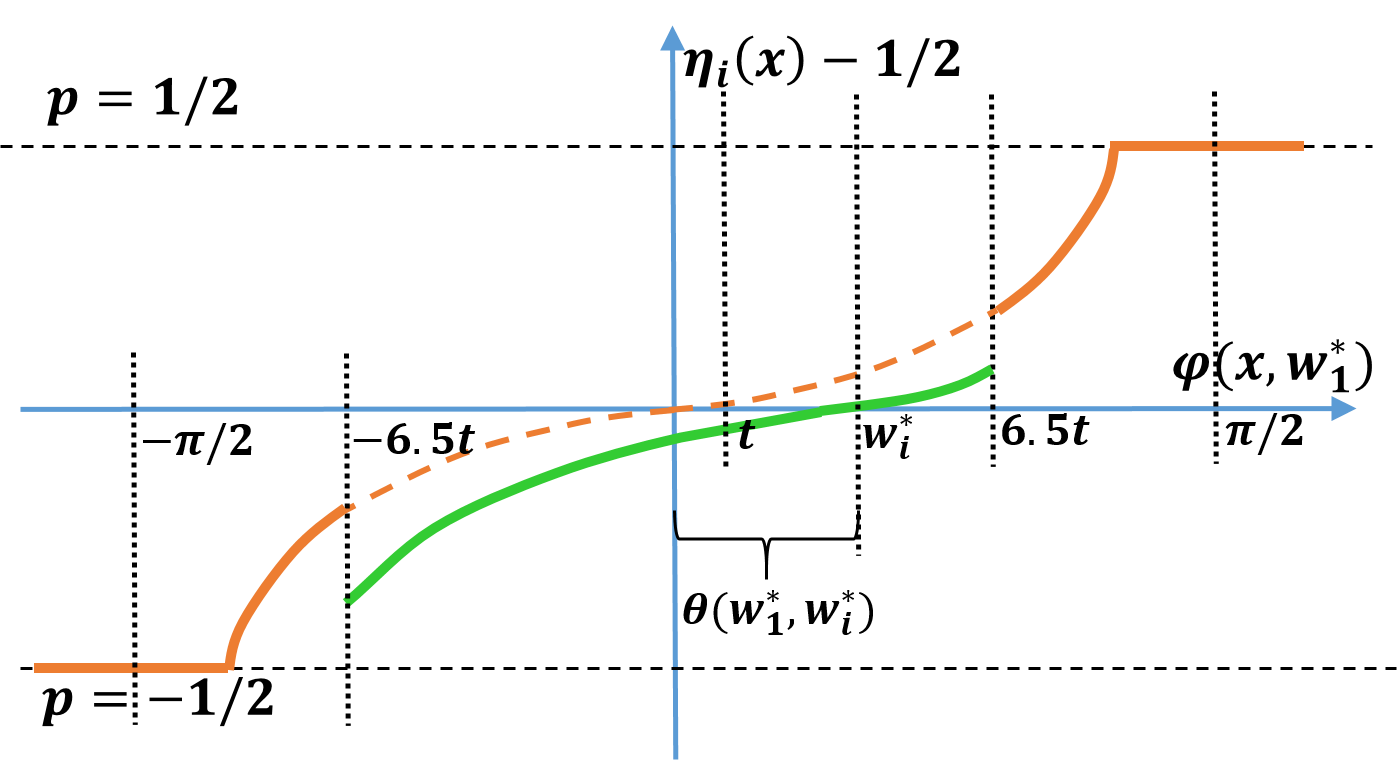}}
\caption{Graphical illustrations of $P_{Y|X}^{(1)}$ (left) and $P_{Y|X}^{(i)}$ (right) constructed as in Eq. (\ref{eq_pyx}).
Solid lines indicate the actual shifted probability density functions $\eta(x)-1/2$ where $\eta(x)=\Pr[Y=1|X=x]$.
In Figure \ref{fig_pi}, the orange curve (both solid and dashed) satisfies TNC with respect to $w_1^*$ and the green curve satisfies TNC with respect to $w_i^*$.
Note the two discontinuities at $\varphi(x,w_1^*) = \pm 6.5t$.
Figure \ref{fig_pi} is not 100\% accurate because it assumes that $\varphi(x,w_1^*)=\varphi(x,w_i^*)+\theta(w_1^*,w_i^*)$,
which may not hold for $d>2$.
}
\label{fig_pyx}
\end{figure}
We use the same distribution when data points are far from the optimal classification hyperplane (i.e., $|\varphi(w^*,x)|>6.5t$)
in order to maximize the ``indistinguishability'' of the constructed conditional distributions.
On the other hand, by TNC assumption $P_{Y|X}^{(i)}$ must have $f_i^*(x)=\sgn(w_i^*\cdot x)$ as its Bayes classifier
and TNC condition along the hyperplane $w_i^*$ must hold.
As a result, when a data point is close to the hyperplane represented by $w_i^*$ the label distribution differs for each hypothesis $w_i^*$ in $\mathcal W$.
Similar construction of adversarial distributions was also adopted in \cite{active-minimax} to prove lower bounds
for one-dimensional active threshold learners.

Lemma \ref{lem_kl_upperbound} summarizes key properties of the label distributions $\{P_{Y|X}^{(i)}\}_{i=1}^m$ constructed as in Eq. (\ref{eq_pyx}).
It is proved in Appendix \ref{appsec:proof_lowerbound} of \cite{mdactive-arxiv}.
\begin{lem}
Suppose $\mathcal W=\{w_1^*,\cdots,w_m^*\}\subseteq\mathbb R^d$ satisfies Eq. (\ref{eq_angle_lowerbound}) and
$\{P_{Y|X}^{(i)}\}_{i=1}^m$ is constructed as in Eq. (\ref{eq_pyx}).
Then for every $i$ the hypothesis $f_i^*(x)=\sgn(w_i^*\cdot x)$ is the Bayes estimator of $P_{Y|X}^{(i)}$ and
the TNC condition in Eq. (\ref{tnc-angle}) holds with respect to $w_i^*$.
In addition, for every $i\neq j$ the KL divergence between $P_{i,T}$ and $P_{j,T}$ is upper bounded by
\begin{equation}
\kl(P_{i,T}\|P_{j,T})\leq C\cdot Tt^{2\alpha/(1-\alpha)},
\label{eq_kl_upperbound}
\end{equation}
where $C$ is a positive constant that does not depend on $T$ or $t$.
\label{lem_kl_upperbound}
\end{lem}

With Lemma \ref{eq_pyx} lower bounding $\theta(w_i^*,w_j^*)$ and Lemma \ref{lem_kl_upperbound} upper bounding $\kl(P_{i,T}\|P_{j,T})$,
Theorem \ref{thm_lower_bound} and Corollary \ref{cor_lower_bound} can be proved by applying standard information theoretical lower bounds \cite{tsybakov-book}.
A complete proof can be found in Appendix \ref{appsec:proof_lowerbound} in \cite{mdactive-arxiv}.

\section{Discussion and remarks}\label{sec:discussion}

\paragraph{Comparison with noise-robust disagreement-based active learning algorithms}
In \cite{hanneke-adaptive} another noise-robust adaptive learning algorithm was introduced.
The algorithm is originally proposed in \cite{agnostic-active-dasgupta}
and is based on the concept of \emph{disagreement coefficient} introduced in \cite{disagreement-coeff}.
The algorithm adapts to different noise level $\alpha$, and achieves an excess error rate of
\begin{equation}
O\left(\left(\frac{\vartheta(d\log T+\log(1/\delta))}{T}\right)^{\frac{1}{2\alpha}}\right)
\end{equation}
with probability $1-\delta$, where $d$ is the underlying dimensionality, 
$T$ is the sample query budget and $\vartheta$ is the disagreement coefficient.
Under our scenario where $X$ is the origin-centered unit ball in $\mathbb R^d$ for $d>2$,
the hypothesis class $\mathbb C$ contains all linear separators whose decision surface passes passing the origin
and $P_{X}$ is the uniform distribution,
the disagreement coefficient $\vartheta$ satisfies \cite{disagreement-coeff}
$\frac{\pi}{4}\sqrt{d} \leq \vartheta\leq \pi\sqrt{d}.$
As a result, the algorithm presented in this paper achieves a polynomial improvement in $d$ in terms of the convergence rate.
Such improvements show the advantage of margin-based active learning
and were also observed in \cite{active-log-concave}.
Also, our algorithm is considerably much simpler and does not require computing lower and upper confidence bounds
on the classification performance.

\paragraph{Connection to adaptive convex optimization}
Algorithm \ref{alg_main} is inspired by an adaptive algorithm for stochastic convex optimization presented in \cite{juditsky-nesterov}.
A function $f$ is called \emph{uniformly convex} on a closed convex set $Q$ if there exists $\rho\geq 2$ and $\mu\geq 0$
such that for all $x,y\in Q$ and $\alpha\in[0,1]$,
\begin{equation}
f(\alpha x+(1-\alpha)y) \leq \alpha f(x) + (1-\alpha)f(y) - \frac{1}{2}\mu\alpha(1-\alpha)\|x-y\|^{\rho}.
\label{uniform-convex}
\end{equation}
Furthermore, if $\mu > 0$ we say the function $f$ is \emph{strongly convex}.
In \cite{juditsky-nesterov} an adaptive stochastic optimization algorithm for uniformly and strongly convex functions was presented.
The algorithm adapts to unknown convexity parameters $\rho$ and $\mu$ in Eq. (\ref{uniform-convex}).

In \cite{alt13} a connection between multi-dimensional stochastic convex optimization and one-dimensional active learning was established.
The TNC condition in Eq. (\ref{tnc-md}) and the strongly convex condition in Eq. (\ref{uniform-convex}) are closely related,
and the exponents $\alpha$ and $\rho$ are tied together in \cite{optimal-rate}.
Based on this connection, a one-dimensional active threshold learner that adapts to unknown TNC noise levels was proposed.

In this paper, we extend the algorithms presented in \cite{juditsky-nesterov,alt13} to build an adaptive margin-based active learning
for multi-dimensional data.
Furthermore, the presented algorithm adapts to all noise level parameters $\alpha\in(0,1)$ with appropriate setting of $r$,
which corresponds to convexity parameters $\rho>1$.
\footnote{The relationship between $\alpha$ and $\rho$ can be made explicitly by noting $\alpha = 1-1/\rho$.}
Therefore,
we conjecture the existence of similar stochastic optimization algorithms that can adapt to a notion of degree of convexity $\rho<2$
as introduced in \cite{alt13}.

\paragraph{Future work}
Algorithm \ref{alg_main} fails to handle the case when $\alpha = 0$.
We feel it is an interesting direction of future work to design active learning algorithms that adapts to $\alpha = 0$
while still retaining the exponential improvement on convergence rate for this case, 
which is observed in previous active learning research \cite{colt07-active,active-log-concave,active-minimax}.

\section*{Acknowledgement}

This research is supported in part by NSF CAREER IIS-1252412.
We would also like to thank Aaditya Ramdas for helpful discussions and Nina Balcan for pointing out an error in an earlier proof.

\bibliography{mdactive}
\bibliographystyle{aaai}

\clearpage
\onecolumn

\appendix

\section{Proof of Theorem \ref{thm_main}}\label{appsec:proof_upperbound}

In this section we give complete proof of Theorem \ref{thm_main} in Section \ref{sec:upperbound}.
We first prove Lemma \ref{lem_opt_passive} and Corollary \ref{cor_opt_passive}.
They analyze the statistical rate of error convergence for passive learning (i.e., per-iteration rate for Algorithm \ref{alg_main}).
The results justify Eq. (\ref{eq_opt_passive}) in Section \ref{sec:upperbound} and can be viewed as a slight generalization of Theorem 4 in \cite{colt07-active}.

\begin{lem}[Optimal passive learning]
Fix $k$.
Let $\widehat w_{k-1}$ be a linear classifier and $b_{k-1}$ be a margin parameter.
Suppose $D_k=\{(x_i,y_i)\}_{i=1}^{n}$ is a training data set of size $n$ that satisfies
$|\widehat w_{k-1}\cdot x_i|\leq b_{k-1}$ for every data point $x_i$.
For some $0<\beta_{k-1}<\frac{\pi}{2}$,
let $\mathcal F_k$ denote the set of all linear classifiers in $B_{\theta}(\widehat w_{k-1}, \beta_{k-1})$ and define
$\widehat w_k= \argmin_{w\in \mathcal F_k}{\widehat{\err}(w|D_k)}$,
$w_k^*= \argmin_{w\in\mathcal F_k}{\err(w|S_1^{(k)})}$,
$\widetilde w_k^*= \argmin_{w\in\mathcal F_k}{\err(w|S_1^{(k)})}$,
where $S_1^{(k)} = \{x||\widehat w_{k-1}\cdot x|\leq b_{k-1}\}$ is the set of all points within a margin of $b_{k-1}$ with respect to $\widehat w_{k-1}$
and $\widehat{\err}(w|D_k) = \frac{1}{|D_k|}\sum_{(x,y)\in D_k}{I[yw\cdot x<0]}$ is the empirical $0/1$ error defined on data set $D_k$.
If $d\geq 4$
and $b_{k-1}\geq \gamma$
with
$$
\gamma := \frac{2\sin\beta_{k-1}}{\sqrt{d}}\sqrt{\ln C+\ln(1+\sqrt{\ln\max(1,\cot\beta_{k-1})})}
$$ 
for some constant $C>0$,
then with probability at least $1-\delta$,
\begin{equation}
\err(\widehat w_k) - \err(w_k^*) \leq \epsilon'\cdot\frac{b_{k-1}\sqrt{d}}{2\sqrt{\pi}} + \frac{2\sin\beta_{k-1}}{C\cdot \cos\beta_{k-1}},
\end{equation}
where $\epsilon'$ satisfies
$$
\epsilon' = \widetilde O\left(\sqrt{\frac{d+1+\log(1/\delta)}{n}}\right).
$$
\label{lem_opt_passive}
\end{lem}

\begin{proof}
Define $S_1^{(k)} := \{x||\widehat w_{k-1}\cdot x|\leq b_{k-1}\}$
and $S_2^{(k)}:= \{x||\widehat w_{k-1}\cdot x|>b_{k-1}\}$.
In Algorithm \ref{alg_main} $D_k\subseteq S_1^{(k)}$.
Since $\theta(\widehat w_{k-1},\widehat w_k), \theta(\widehat w_{k-1},w_k^*)\leq\beta_{k-1}$,
by Lemma 7 in \cite{colt07-active}, we have
\begin{eqnarray*}
\Pr[(\widehat w_{k-1}\cdot x)(\widehat w_k\cdot x)<0,x\in S_2^{(k)}]&\leq& \frac{\sin\beta_{k-1}}{C\cos\beta_{k-1}},\\
\Pr[(\widehat w_{k-1}\cdot x)(w_k^*\cdot x)<0,x\in S_2^{(k)}]&\leq& \frac{\sin\beta_{k-1}}{C\cos\beta_{k-1}}.
\end{eqnarray*}
Adding the two inequalities we get
\begin{equation*}
\Pr[(\widehat w_k\cdot x)(w_k^*\cdot x)<0, x\in S_2^{(k)}] \leq \frac{2\sin\beta_{k-1}}{C\cos\beta_{k-1}}.
\end{equation*}
Subsequently,  we have
\begin{equation}
\err(\widehat w_k)-\err(w_k^*) 
{\leq} (\err(\widehat w_k|S_1^{(k)})-\err(w_k^*|S_1^{(k)}))\Pr[x\in S_1^{(k)}]+ \frac{2\sin\beta_{k-1}}{C\cos\beta_{k-1}}.
\label{eq_passive_sample_complexity}
\end{equation}

Applying standard sample complexity bounds (e.g., Theorem 8 in \cite{colt07-active}) we have
$$
\err(\widehat w_k|S_1^{(k)})-\err(w_k^*|S_1^{(k)}) = \widetilde O\left(\sqrt{\frac{d+1+\log(1/\delta)}{|D_k|}}\right)
$$
with probability $\geq 1-\delta$.
The proof is then completed by the fact that $\Pr[x\in S_1^{(k)}]\leq \frac{b_{k-1}\sqrt{d}}{2\sqrt{\pi}}$, as shown in Lemma 4 in \cite{colt07-active}.

\end{proof}

\begin{cor}
Fix $k$.
Let $T$ be the number of samples obtained and $E$ be the number of iterations.
Suppose $d\geq 4$,
$T\geq 4$,
$E=\frac{1}{2}\log T$,
$n = T/E$.
Assume also that $\beta_k = r^k\pi$ and 
$b_{k-1} = \frac{2\beta_{k-1}}{\sqrt{d}}\sqrt{E(1+\log(1/r))}$
for some constant $r\in(0,1/2)$.
With probability at least $1-\delta$,
\begin{equation}
\err(\widehat w_k)-\err(w_k^*)
\leq \beta_{k-1}\cdot\left[\epsilon'\sqrt{E(1+\log(1/r))} + 2\sqrt{\frac{2}{T}}\right]
=: \beta_{k-1}\epsilon.
\label{eq_opt_passive}
\end{equation}
\label{cor_opt_passive}
\end{cor}
\begin{proof}
We first note that Eq.~(\ref{eq_opt_passive}) is trivially true for $k=1$ because 
$\beta_0=\pi>1$, $\sqrt{E(1+\log(1/r))}\geq 1$ and 
$\err(\widehat w_1)-\err(w_1^*) \leq\epsilon'$, following Eq. (\ref{eq_passive_sample_complexity}).
In the remaining part of the proof we will assume that $k\geq 2$, and hence $\beta_{k-1}\in(0,\pi/2)$.

Before continuing the proof, we first give a list of elementary inequalities
that will be used in this proof (assuming $\beta<\pi/2$):
\begin{equation*}
\ln(1+x)\leq x;\quad \sin\beta\leq \beta;\quad \cot\beta\leq 1/\beta;\quad \tan\beta\leq \sqrt{2}\beta.
\end{equation*}

Put $C=2^E = \sqrt{T}$. 
Let's first check $b_{k-1}\geq\gamma$:
\begin{eqnarray*}
\gamma &=& \frac{2\sin\beta_{k-1}}{\sqrt{d}}\sqrt{\ln C + \ln(1+\sqrt{\ln\max(1,\cot\beta_{k-1})})}\\
&\leq& \frac{2\beta_{k-1}}{\sqrt{d}}\sqrt{\ln C + \sqrt{\ln\max(1,1/\beta_{k-1})}}\\
&\leq& \frac{2\beta_{k-1}}{\sqrt{d}}\sqrt{E + E\log(1/r)} =: b_{k-1}.
\end{eqnarray*}

Applying Lemma \ref{lem_opt_passive}, we have
\begin{eqnarray*}
& &\err(\widehat w_k)-\err(w_k^*)\\
&\leq& \epsilon'\frac{\beta_{k-1}}{\sqrt{\pi}}\sqrt{E(1+\log(1/r))} + \frac{2\sqrt{2}\beta_{k-1}}{\sqrt{T}}\\
& \leq & \beta_{k-1}\cdot\left[\epsilon'\sqrt{E(1+\log(1/r))} + 2\sqrt{\frac{2}{T}}\right]
= \beta_{k-1}\epsilon.
\end{eqnarray*}
\end{proof}

With Lemma \ref{lem_opt_passive} and Corollary \ref{cor_opt_passive} we are ready to prove the main theorem.
The key component of the proof is a case analysis on the TNC parameter $\mu$.
By Eq. (\ref{tnc-md}) and the fact that $0<\alpha < 1$, $\mu$ must satisfy $\mu\leq \pi^{-1/(1-\alpha)} \leq \pi^{-\alpha/(1-\alpha)}$.
Since $\alpha\geq 1/(1+\log(1/r))$, we have the following upper bound for $\mu$:
\begin{equation}
\mu\leq \frac{2^{(\log(1/r)\cdot\frac{\alpha}{1-\alpha}-1)E}\cdot r^{-1}}{\pi^{\frac{\alpha}{1-\alpha}}} 
= \frac{2^{-E}r^{-1}}{(r^E\pi)^{\alpha/(1-\alpha)}} 
=  \frac{2^{-E}r^{-1}}{\beta_E^{\alpha/(1-\alpha)}} \leq \frac{r^{-1}\epsilon}{\beta_E^{\alpha/(1-\alpha)}}.
\label{eq_mu_upperbound}
\end{equation}
The last step holds due to $\epsilon \geq 1/\sqrt{T} = 2^{-E}$.

When $\mu$ is sufficiently small, that is, 
\begin{equation}
\mu\leq \frac{r^{-1}\epsilon}{\beta_1^{\alpha/(1-\alpha)}},
\label{eq_mu_small}
\end{equation}
we are done after the first iteration because
\begin{equation}
\err(\widehat w_1)-\err(w_1^*)
\leq \beta_0 \epsilon 
\leq r^{-1}\beta_1\epsilon\\
\leq r^{-1}\left(\frac{\epsilon}{r\mu}\right)^{(1-\alpha)/\alpha}\epsilon
 = \frac{\epsilon^{1/\alpha}}{r^{1/\alpha}\mu^{(1-\alpha)/\alpha}}
\label{eq_small_mu_phase1}
\end{equation}
and with probability $\geq 1-\delta E$ we have $w_1^*=w^*$ and using 
arguments akin to Lemma~\ref{lem3} in Section \ref{sec:upperbound}, we have
\footnote{Note that in this case $\mu< r^{-1}\epsilon/\beta_1^{\alpha/(1-\alpha)}$ for all $k$.}
\begin{eqnarray}
\err(\widehat w_E)-\err(\widehat w_1)
\leq \frac{r}{1-r}\pi\epsilon.\label{eq_small_mu_phase2}
\end{eqnarray}
Combining Eq. (\ref{eq_small_mu_phase1}) and (\ref{eq_small_mu_phase2}) we prove the theorem.
Next we consider the case when Eq. (\ref{eq_mu_small}) does not hold.
In this case, by Eq. (\ref{eq_mu_upperbound}) there must exists $k^*\in\{1,2,\cdots,E-1\}$ such that
\begin{equation}
\frac{r^{-1}\epsilon}{\beta_{k^*}^{\alpha/(1-\alpha)}}
\leq \mu
\leq \frac{r^{-1}\epsilon}{\beta_{k^*+1}^{\alpha/(1-\alpha)}}.
\label{eq_case_mu}
\end{equation}

The $k^*$ in Eq. (\ref{eq_case_mu}) defines the ``tipping-point" of Algorithm \ref{alg_main} and is crucial to the three technical lemmas
presented in Section \ref{sec:sketch_upperbound}.
Combining Lemma \ref{lem1}, \ref{lem2} and \ref{lem3} we get
\begin{eqnarray*}
&&\err(\widehat w_E)-\err(w^*)  \\
&=& \err(\widehat w_E)-\err(\widehat w_{k^*}) + \err(\widehat w_{k^*})-\err(w^*) \\
&=& \err(\widehat w_E)-\err(\widehat w_{k^*}) + \err(\widehat w_{k^*})-\err(w^*_{k^*}) \\
&\leq& \frac{r}{1-r} \beta_{k^*-1}\epsilon + \beta_{k^*-1}\epsilon\\
&=& \frac{1}{1-r} \beta_{k^*-1}\epsilon \leq \frac{\epsilon^{1/\alpha}}{r^{\frac{1+\alpha}{\alpha}}(1-r)\mu^{\frac{1-\alpha}{\alpha}}};
\end{eqnarray*}
thus, completing proof of the main theorem. 

\subsection{Proofs of technical lemmas in Section \ref{sec:sketch_upperbound}}

\begin{proof}[Proof of Lemma \ref{lem1}]
We know that 
$$\beta_{k^*-1} = r^{-2}\beta_{k^*+1} \leq r^{-2}\left(\frac{\epsilon}{r\mu}\right)^{\frac{1-\alpha}{\alpha}}.$$
Applying Corollary \ref{cor_opt_passive} we get
\begin{equation*}
\err(\widehat w_{k^*})-\err(w_{k^*}^*) \leq  \beta_{k^*-1}\epsilon \\
\leq r^{-2}\left(\frac{\epsilon}{r\mu}\right)^{\frac{1-\alpha}{\alpha}} \epsilon = \frac{\epsilon^{1/\alpha}}{r^{\frac{1+\alpha}{\alpha}}\mu^{\frac{1-\alpha}{\alpha}}}.
\end{equation*}
\end{proof}

\begin{proof}[Proof of Lemma \ref{lem2}]
{
Note that in order to prove $w_k^* = w^*$ we only need to show $\theta(w^*, \widehat w_{k-1})\leq \beta_{k-1}$ 
because the global optimal classifier $w^*$ is the Bayes estimator and is hence optimal for any region $S\subseteq X$.

We now use induction to prove this lemma.
When $k=1$ the claim $w_1^* = w^*$ is clearly true because $\beta_0 = \pi$.
Now assume the claim is true for $k$, that is, $\theta(\widehat w_{k-1},w^*)\leq \beta_{k-1}$ and $w_k^* = w^*$.
We want to prove $\theta(\widehat w_k,w^*)\leq\beta_k$, i.e., $w_{k+1}^* = w^*$.
}

To see this, we apply the TNC condition in Eq. (\ref{tnc-md}) to get
\begin{equation*}
\theta(\widehat w_k,w^*) \leq \left(\frac{\err(\widehat w_k)-\err(w^*)}{\mu}\right)^{1-\alpha} 
\end{equation*}

Therefore, we only need to prove 
$\mu\geq (\err(\widehat w_k)-\err(w^*))/\beta_k^{1/(1-\alpha)} .$

{
Because of the induction assumption $\theta(\widehat w_{k-1}, w^*)\leq \beta_{k-1}$, 
$w^*$ is exactly $w_{k}^*$.
Applying Corollary \ref{cor_opt_passive} 
we know that with probability at least $1-\delta$,
\begin{equation*}
\frac{\err(\widehat w_k)-\err(w^*)}{\beta_k^{1/(1-\alpha)}} 
\leq \frac{\beta_{k-1}\epsilon}{\beta_k^{1/(1-\alpha)}}
=\frac{r^{-1}\epsilon}{\beta_k^{\alpha/(1-\alpha)}} \leq \frac{r^{-1}\epsilon}{\beta_{k^*}^{\alpha/(1-\alpha)}}.
\end{equation*}
The last expression is $\leq \mu$ by applying Eq. (\ref{eq_case_mu}), and hence $\theta(\widehat w_k, w^*) \leq \beta_k$.
Finally, taking a union bound over all $k \leq k^*$ we complete the proof.
}

\end{proof}

\begin{proof}[Proof of Lemma \ref{lem3}]
First, decompose the term $\err(\widehat w_E) - \err(\widehat w_{k^*})$ as follows:
\begin{align*}
&{}\err(\widehat w_E) - \err(\widehat w_{k^*}) \\
&{=} \sum_{k=k^*}^{E-1}{\err(\widehat w_{k+1}) - \err(\widehat w_k)}\\
&{=} \sum_{k=k^*}^{E-1}{(\err(\widehat w_{k+1}|S_1^{(k)}) - \err(\widehat w_k|S_1^{(k)}))\Pr[x\in S_1^{(k)}]}\\
&{}+ (\err(\widehat w_{k+1}|S_2^{(k)}) - \err(\widehat w_k|S_2^{(k)}))\Pr[x\in S_2^{(k)}]\\
&{\leq} \sum_{k=k^*}^{E-1}{(\err(\widehat w_{k+1}|S_1^{(k)}) - \err(w_{k+1}^*|S_1^{(k)}))\Pr[x\in S_1^{(k)}]}\\
&{}+ \Pr[(\widehat w_k\cdot x)(\widehat w_{k+1}\cdot x) < 0, x\in S_2^{(k)}],
\end{align*}
where $S_1^{(k)} = \{x||\widehat w_k\cdot x|\leq b_k\}$ and $S_2^{(k)} = \{x||\widehat w_k\cdot x|\geq b_k\}$.
The last inequality is due to the fact that $w_{k+1}^*$ is the optimal classifier with respect to the conditional distribution on $S_1^{(k)}$.
Next, taking a union bound over all rounds $k^* \leq k \leq E$ and using arguments akin to Lemma \ref{lem_opt_passive} and Corollary \ref{cor_opt_passive},
we have with probability at least $1-\delta E$
\begin{equation*}
\err(\widehat w_E) - \err(\widehat w_{k^*})
\leq \sum_{k=k^*}^{E-1}{\beta_k\epsilon}
\leq \beta_{k^*-1}\epsilon(r+r^2+\cdots)
= \frac{r}{1-r}\beta_{k^*-1}\epsilon.
\end{equation*}
\end{proof}

\subsection{Extension to log-concave data distributions}\label{appsec:logconcave}

In this section we prove Theorem \ref{thm_main_logconcave}.
In particular, we show that Algorithm \ref{alg_main} works if the data distribution $P_{X}$ is isotropic log-concave
and the algorithm parameters are chosen as $\beta_k=r^k\pi$ and $b_{k-1}=C_1\beta_{k-1}\log T$
for some absolute constant $C_1$.

Let $g:\mathbb R^d\to (0,+\infty)$ be the pdf of the underlying data distribution $P_X$.
We say $P_X$ is a \emph{log-concave} distribution if $\log g$ is a concave function.
Furthermore, $P_X$ is \emph{isotropic} if its mean is zero and its covariance matrix is the identity.

For isotropic log-concave densities, we cite the following lemma from \cite{active-log-concave}.
\begin{lem}[Theorem 21, \cite{active-log-concave}]
There exist absolute constants $C_1, C_2, C_3 > 0$ such that the following holds.
Let $u$ and $v$ be two unit vectors in $\mathbb R^d$ and assume $\theta(u, v) = \eta < \pi/2$.
Assume that $D$ is isotropic log-concave, then for any $b \geq C_1\eta$ we have
\begin{equation}
\Pr_{x\sim D}{[(u\cdot x)(v\cdot x) < 0, |v\cdot x|\geq b]} \leq C_2\eta\exp(-C_3b/\eta).
\label{eq_marginerror_logconcave}
\end{equation}
\label{lem_marginerror_logconcave}
\end{lem}

Using Lemma \ref{lem_marginerror_logconcave} we can derive a result similar to Corollary \ref{cor_opt_passive}, as shown below.
\begin{lem}
Fix $k$ and let $C_1, C_2, C_3$ be absolute constants in Lemma \ref{lem_marginerror_logconcave}.
Let $T$ be the number of samples obtained and $E$ be the number of iterations.
Suppose $d \geq 2$, $T\geq 4$, $E = \frac{1}{2}\log T$, $n = T/E$.
Assume also that $\beta_k = r^k\pi$ for some constant $r\in(0,1/2)$ and $b_{k-1} = C_1\beta_{k-1}\log T$ for constant $C_1$ in Lemma \ref{lem_marginerror_logconcave}.
Suppose $D_k = \{(x_i,y_i)\}_{i=1}^n$ is a training data set of size $n$ 
and $(x_i,y_i)$ are i.i.d. sampled from $P_X$.
In addition, all $(x_i,y_i)$ in $D_k$ satisfy $|\widehat w_{k-1}\cdot x_i|\leq b_{k-1}$ for some margin classifier $\|\widehat w_{k-1}\| = 1$.
If $P_X$ is isotropic log-concave, then with probability at least $1-\delta$ the following holds:
\begin{equation}
\err(\widehat w_k) - \err(w_k^*) 
\leq \beta_{k-1}\left( \epsilon'\cdot 2C_1\log T + 2C_2T^{-C_1C_3}\right) =: \beta_{k-1}\epsilon,
\label{eq_opt_passive_logconcave}
\end{equation}
{
where $\widehat w_{k-1} = \argmin_{w\in\mathcal F_k}{\widehat{\err}(w|D_k)}$,
$w_k^* = \argmin_{w\in\mathcal F_k}{\err(w|S_1^{(k)})}$,
$\mathcal F_k = B_{\theta}(\widehat w_{k-1}, \beta_{k-1})$,
$S_1^{(k)} = \{x||\widehat w_{k-1}\cdot x|\leq b_{k-1}\}$.
}
\label{lem_opt_passive_logconcave}
\end{lem}
\begin{proof}
Define $S_1^{(k)} := \{x||\widehat w_{k-1}\cdot x|\leq b_{k-1}\}$ and $S_2^{(k)} := \{x||\widehat w_{k-1}\cdot x| > b_{k-1}\}$.
By Algorithm \ref{alg_main} $D_k \subseteq S_1^{(k)}$.
Since $\theta(\widehat w_{k-1}, \widehat w_k), \theta(\widehat w_{k-1}, w_k^*) \leq \beta_{k-1}$, by Lemma \ref{lem_marginerror_logconcave} we have
\begin{eqnarray*}
\Pr[(\widehat w_{k-1}\cdot x)(\widehat w_k\cdot x) < 0, x\in S_2^{(k)}] &\leq& C_2\beta_{k-1}T^{-C_1C_3},\\
\Pr[(\widehat w_{k-1}\cdot x)(w_k^*\cdot x) < 0, x\in S_2^{(k)}] &\leq& C_2\beta_{k-1}T^{-C_1C_3}.
\end{eqnarray*}
Adding the two inequalities we get 
\begin{equation*}
\Pr[(\widehat w_k\cdot x)(w_k^*\cdot x) < 0, x\in S_2^{(k)}] \leq 2C_2\beta_{k-1}T^{-C_1C_3}.
\end{equation*}

Next, let $\widetilde w_k^*$ denote the minimizer of the true surrogate loss error $\err$ on $S_1$ 
over all linear classifiers with $w\in B_{\theta}(\widehat w_{k-1}, \beta_{k-1})$.
Consequently, 
\begin{eqnarray*}
& &\err(\widehat w_k) - \err(w_k^*)\\
&\leq& (\err(\widehat w_k|S_1^{(k)}) - \err(w_k^*|S_1^{(k)}))\Pr[x\in S_1^{(k)}] 
+ 2C_2\beta_{k-1}T^{-C_1C_3}\\
&\leq& \epsilon'\cdot 2b_k + 2C_2\beta_{k-1}T^{-C_1C_3},
\end{eqnarray*}
where the last inequaqlity is due to Lemma 2 in \cite{active-log-concave}.

\end{proof}

\begin{proof}[Proof of Theorem \ref{thm_main_logconcave}]
The proof is the same as the proof of Theorem \ref{thm_main} by noting that $\epsilon=\epsilon'\cdot O(\log T)$
is of the order $\widetilde O(\sqrt{d/T})$.
In particular, Lemma \ref{lem1},\ref{lem2} and \ref{lem3} do not require the uniform distribution assumption.
\end{proof}

\section{Proof of Theorem \ref{thm_lower_bound}}\label{appsec:proof_lowerbound}

In this section we provide a complete proof of Theorem \ref{thm_lower_bound}.
We first prove the main technical lemmas (Lemma \ref{lem_angle_lowerbound} and \ref{lem_kl_upperbound}) in Section \ref{sec:sketch_lowerbound}.

\begin{proof}[Proof of Lemma \ref{lem_angle_lowerbound}]

We first present the following lemma, which will be proved in Appendix \ref{appsec:proof_lowerbound_technical}.
\begin{lem}
Assume $d$ is even.
Let $S$ be the largest subset of $\{0,1\}^d$ that satisfies the following conditions:
\begin{enumerate}
\item $\forall x\in S$, $\|x\|_1 = d/2 =: \omega$.
\item $\forall x,x'\in S$, $x\neq x'$, $\Delta_H(x,x') \geq d/16 =: \delta_H$.
\end{enumerate}
Here $\Delta_H(x,y)=\sum_i{|x_i-y_i|}$ denotes the Hamming distance between $x$ and $y$.
Then for $d\geq 2$, the following lower bound on the size of $S$ holds:
\begin{equation}
\log |S| \geq 0.0625d.
\label{eq_lb_S}
\end{equation}
\label{lem_const_weight_coding_cor}
\end{lem}

Using Lemma \ref{lem_const_weight_coding_cor} we can construct a well-separated hypothesis set that satisfies Eq. (\ref{eq_angle_lowerbound}) as follows.
Suppose $S=\{z_1,\cdots,z_m\}\subseteq\{0,1\}^d$ is the largest subset of $\{0,1\}^d$ that satisfies the two conditions in Lemma \ref{lem_const_weight_coding_cor}.
Set $a=4t$. For each $i\in\{1,2,\cdots,m\}$ define $w_i^*\in\mathbb R^d$ as
\begin{equation}
w_i^* = \frac{1}{Z}\left((1,1,\cdots,1) - az_i\right),
\label{eq_w}
\end{equation}
where $Z=\sqrt{d\left(1-a+\frac{a^2}{2}\right)}$ is a normalization constant to make $\|w_i^*\|=1$.
Note that $Z$ can be upper  and lower bounded by $\frac{1}{2}d\leq Z^2\leq d$ because $\|z_i\|_1=\frac{d}{2}$.
Define $\mathcal W=\{w_1^*,\cdots,w_m^*\}$.
Since $|\mathcal W|=|S|$, by Lemma \ref{lem_const_weight_coding_cor} 
we have $\log |\mathcal W|\geq 0.0625d$ for $d\geq 2$.

Next we prove that $\mathcal W$ is a well-separated hypothesis set that satisfies Eq. (\ref{eq_angle_lowerbound}).
That is, for every pair $w_i^*,w_j^*\in\mathcal W,i\neq j$ one has $t\leq\theta(w_i^*,w_j^*)\leq 6.5t$.
First consider the lower bound.
A key observation is that two mismatched entries in $z_i$ and $z_j$ reduces the inner product between $w_i^*$ and $w_j^*$ by $a^2$, as shown in Eq. (\ref{eq_illustr}).
{Note that such ``switching" is always possible because both $z_i$ and $z_j$ have the same number of ones and hence the Hamming distance $\Delta_H(z_i,z_j)$ is even.}
\begin{equation}
2-2a+a^2 = \left\lgroup
\begin{array}{ccc} 
1-a& & 1\\
\times& +& \times\\
1-a& & 1\\\end{array}\right\rgroup
\Longrightarrow
\left\lgroup
\begin{array}{ccc}
1-a& & 1\\
\times& +& \times\\
1& & 1-a\\\end{array}\right\rgroup
= 2-2a.
\label{eq_illustr}
\end{equation}
By definition of $w_i^*$ and $w_j^*$ we have
\begin{equation}
\cos\theta(w_i^*,w_j^*) = \langle w_i,w_j\rangle = \frac{d\left(1-a+\frac{a^2}{2}\right) - \frac{\Delta_H(z_i,z_j)}{2}a^2}{d\left(1-a+\frac{a^2}{2}\right)}
\leq 1-\frac{a^2}{32}.
\end{equation}
Subsequently, using the fact that $\cos\theta \geq 1-\frac{\theta^2}{2}$ we have
\begin{equation}
\theta(w_i^*,w_j^*) \geq \sqrt{2(1-\cos\theta(w_i^*,w_j^*))} \geq \frac{a}{4} = t.
\end{equation}

On the other hand, to obtain an upper bound on $\theta(w_i^*,w_j^*)$, we note the following lower bound on $\cos\theta(w_i^*,w_j^*)$:
\begin{equation}
\cos\theta(w_i^*,w_j^*) = \langle w_i,w_j\rangle = \frac{d\left(1-a+\frac{a^2}{2}\right)-\frac{\Delta_H(z_i,z_j)}{2}a^2}{d\left(1-a+\frac{a^2}{2}\right)}
\geq 1-\frac{a^2/2}{1/2} = 1-a^2.
\end{equation}
The second to last inequality is due to the fact that $\Delta_H(z_i,z_j)\leq d$ and $Z^2\geq \frac{d}{2}$.
By Taylor expansion one has $\cos\theta \leq 1-\frac{\theta^2}{2}+\frac{\theta^4}{24} \leq 1-\frac{\theta^2}{2}+\theta^2\frac{\theta^2}{24}$.
Since $\langle w_i,w_j\rangle > 0$ the angle $\theta$ must be smaller than $\frac{\pi}{2}$.
Consequently, $\cos\theta\leq 1-\frac{7}{18}\theta^2$ and therefore we have the following upper bound on $\theta(w_i^*,w_j^*)$:
\begin{equation}
\theta(w_i^*,w_j^*) \leq \sqrt{\frac{18}{7}(1-\cos(\theta(w_i^*,w_j^*)))} \leq \sqrt{\frac{18}{7}}a \leq 6.5t.
\end{equation}

\end{proof}

\begin{proof}[Proof of Lemma \ref{lem_kl_upperbound}]
We first prove that $w_i^*$ is the Bayes classifier for $P_{Y|X}^{(i)}$ and furthermore
$P_{Y|X}^{(i)}$ satisfies the TNC condition in Eq. (\ref{tnc-angle}) with respect to $w_i^*$.
note that $\eta(x)=\frac{1}{2}$ if and only if $|\varphi(x,w_1^*)|\leq 6.5t$ and $\varphi(x,w_i^*)=0$.
Therefore, $w_i^*$ is the Bayes classifier for $P_{Y|X}^{(i)}$.
We also note that if $|\varphi(x,w_1^*)|\geq 6.5t$ then $\sgn(\varphi(x,w_1^*))=\sgn(\varphi(x,w_i^*))$ 
because $\theta(w_1^*,w_i^*)\leq 6.5t$.
Without loss of generality, we assume in the remainder of the proof that $\varphi(x,w_1^*), \varphi(x,w_i^*)\geq 0$ when $|\varphi(x,w_1^*)|>6.5t$.

Next we prove that $P_{Y|X}^{(i)}$ satisfies the TNC condition with respect to $w_i^*$.
Consider a data point $x$ with $\varphi(x,w_1^*) = \vartheta$ and $\varphi(x,w_i^*) = \vartheta'$.
By definition, if $|\vartheta|\leq 6.5t$ then $P_{Y|X}^{(i)}$ trivially satisfies TNC because $P_{Y|X}^{(i)}(Y=1|X=x)$
only depends on $\vartheta'$.
For $|\vartheta|>6.5t$ and $\vartheta,\vartheta'>0$, note that $0.5\vartheta'\leq \vartheta$ because $|\vartheta-\vartheta'|\leq 6.5t$
and hence $\vartheta' \leq \vartheta+6.5t < 2\vartheta$.
Consequently, $P_{Y|X}^{(i)}$ satisfies
\begin{equation}
P_{Y|X}^{(i)}(Y=1|X=x)-\frac{1}{2} = \min\{1/2, 2^{\frac{\alpha}{1-\alpha}}\mu_0\vartheta^{\frac{\alpha}{1-\alpha}}\}
\geq \min\{1/2, \mu_0\vartheta'^{\frac{\alpha}{1-\alpha}}\} \geq \mu_0\vartheta'^{\frac{\alpha}{1-\alpha}}.
\end{equation}
The last inequality holds due to the fact that Eq. (\ref{tnc-angle}) holds for all $\vartheta\in[0,\pi]$.
Therefore, $P_{Y|X}^{(i)}$ satisfies the TNC lower bound for $|\vartheta|>6.5t$ with respect to $w_i^*$.

Lastly, we prove the upper bound on $\kl(P_{i,T}\|P_{j,T})$ in Eq. (\ref{eq_kl_upperbound}).
Here $P_{i,T}$ represents the data/label distribution for an active learning algorithm that is allowed for $T$ label queries
 under $P_{Y|X}^{(i)}$.
Mathematically, $P_{i,T}$ is a distribution over $(\mathcal X\times\mathcal Y)^T$ and can be decomposed as
\begin{multline*}
P_{i,T}(x_1,y_1,\cdots,x_T,y_T) 
= P_{X_1,Y_1,\cdots,X_T,Y_T}^{(i)}(x_1,y_1,\cdots,x_T,y_T)\\
= \prod_{t=1}^T{P_{Y_t|X_t}^{(i)}(y_t|x_t)P_{X_t|X_1,Y_1,\cdots,X_{t-1},Y_{t-1}}(x_t|x_1,y_1,\cdots,x_{t-1},y_{t-1})}
\end{multline*}
because a query synthetic algorithm $A\in\mathcal A_{d,T}^{\prob}$ proposes data points and requests labels in a sequential, feedback-driven manner.
Note that $P_{X_t|X_1,Y_1,\cdots,X_{t-1},Y_{t-1}}$ does not depend on $i$ because the underlying label distribution $P_{Y|X}^{(i)}$
is unknown to $A$.
Subsequently, we establish an upper bound on $\kl(P_{i,T}\|P_{j,T})$ following analysis in \cite{active-minimax}:
($\mathbb E_i$ and $\mathbb E_j$ denote the expectation taking with respect to $P_{i,T}$ and $P_{j,T}$)
\begin{eqnarray}
\kl(P_{i,T}\|P_{j,T})
&=& \mathbb E_i\left[\log\frac{P_{X_1,Y_1,\cdots,X_T,Y_T}^{(i)}(x_1,y_1,\cdots,x_T,y_T)}{P_{X_1,Y_1,\cdots,X_T,Y_T}^{(j)}(x_1,y_1,\cdots,x_T,y_T)}\right]\nonumber\\
&=& \mathbb E_i\left[\log\frac{\prod_{t=1}^T{P_{Y_t|X_t}^{(i)}(y_t|x_t)P_{X_t|X_1,Y_1,\cdots,X_{t-1},Y_{t-1}}(x_t|x_1,y_1,\cdots,x_{t-1},y_{t-1})}}{\prod_{t=1}^T{P_{Y_t|X_t}^{(j)}(y_t|x_t)P_{X_t|X_1,Y_1,\cdots,X_{t-1},Y_{t-1}}(x_t|x_1,y_1,\cdots,x_{t-1},y_{t-1})}}\right]\nonumber\\
&=& \mathbb E_i\left[\log\frac{\prod_{t=1}^T{P_{Y_t|X_t}^{(i)}(y_t|x_t)}}{\prod_{t=1}^T{P_{Y_t|X_t}^{(j)}(y_t|x_t)}}\right]\nonumber\\
&=& \sum_{t=1}^T{\mathbb E_i\left[\mathbb E_i\left[\log\frac{P_{Y|X}^{(i)}(y_t|x_t)}{P_{Y|X}^{(j)}(y_t|x_t)}\Bigg| X_1=x_1,\cdots,X_T=x_T\right]\right]}\nonumber\\
&\leq& T\cdot\sup_{x\in\mathcal X}\kl(P_{Y|X}^{(i)}(\cdot|x)\|P_{Y|X}^{(j)}(\cdot|x)).
\label{eq_ub_kl}
\end{eqnarray}
Eq. (\ref{eq_ub_kl}) shows that the upper bound on $\kl(P_{i,T}\|P_{j,T})$ does not depend on which active learning algorithm is used,
though both $P_{i,T}$ and $P_{j,T}$ are defined in an algorithm-dependent way.
Note that both $P_{Y|X}^{(i)}(\cdot|x)$ and $P_{Y|X}^{(j)}(\cdot|x)$ are Bernoulli random variables.
To bound their KL divergence we cite the following result from \cite{active-minimax}:
\begin{lem}[\cite{active-minimax}, Lemma 1]
Let $P$ and $Q$ be Bernoulli random variables with parameters $1/2-p$ and $1/2-q$, respectively.
If $|p|,|q|\leq 1/4$ then $\kl(P\|Q)\leq 8(p-q)^2$.
\label{lem_bernoulli_kl}
\end{lem}

Fix $x\in\mathcal X=\mathcal S^d$ and let $\vartheta_i = \varphi(x,w_i^*)$, $\vartheta_j=\varphi(x,w_j^*)$ and $\vartheta_1 = \varphi(x,w_1^*)$.
Suppose $P_{Y|X}^{(i)}(\cdot|x) = 1/2-p_i$ and $P_{Y|X}^{(j)}(\cdot|x)=1/2-p_j$.
A simple case study on $\vartheta_1$ reveals that $|p_i-p_j|=0$ when $|\vartheta_1|>6.5t$ and
\begin{equation}
\big|p_i-p_j\big|\leq c\cdot\left(\big|\vartheta_i\big|^{\alpha/(1-\alpha)}+\big|\vartheta_j\big|^{\alpha/(1-\alpha)}\right)\leq C\cdot t^{\alpha/(1-\alpha)}
\end{equation}
when $|\vartheta_1|\leq 6.5t$ (which implies $|\vartheta_i|,|\vartheta_j|\leq 13t$).
Here $c=2^{\alpha/(1-\alpha)}\mu_0$ and $C= 2\times 13^{\alpha/(1-\alpha)}c$ are constants that do not depend on $t$ or $T$.
Furthermore, for sufficiently small $t$ one has $|p_i|,|p_j|\leq 1/4$.
Therefore, by Lemma \ref{lem_bernoulli_kl} the following holds:
\begin{equation}
\kl(P_{i,T}\|P_{j,T}) \leq T\cdot\sup_{x\in\mathcal S^d}\kl(P_{Y|X}^{(i)}(\cdot|x)\|P_{Y|X}^{(j)}(\cdot|x)) \leq 8C^2\cdot  Tt^{2\alpha/(1-\alpha)}.
\end{equation}
\end{proof}

To prove Theorem \ref{thm_lower_bound}, we cite the following information-theoretical lower bound from \cite{tsybakov-book}.
\begin{thm}[\cite{tsybakov-book}]
Let $\mathcal F$ be a set of models and $\mathcal F_0=\{f_1,\cdots,f_M\}\subseteq\mathcal F$
be a finite subset of $\mathcal F$.
We have a probability measure $P_f$ defined on a common probability space associated with each model $f\in\mathcal F$.
Let $D:\mathcal F\times\mathcal F\to\mathbb R$ be a collection of semi-distances.
If there exist constants $\rho>0$ and $0<\gamma<1/8$ such that the following holds:
\begin{enumerate}
\item $D(f_j,f_k)\geq 2\rho>0$ for every $j,k\in\{1,\cdots,M\}$, $j\neq k$.
\item $P_{f_j}\ll P_{f_1}$ for every $j\in\{1,\cdots,M\}$.
\footnote{For two distributions $P$ and $Q$, $P\ll Q$ means the support of $P$ is contained in the support of $Q$.}
\item $\frac{1}{M}\sum_{j=1}^{M}{\kl(P_{f_j}\|P_{f_0})} \leq \gamma\log M$.
\end{enumerate}
Then the following bound holds:
\begin{equation}
\inf_{\widehat f}{\sup_{f\in\mathcal F}}P_f\left(D(\widehat f,f)\geq \rho\right)
\geq \frac{\sqrt{M}}{1+\sqrt{M}}\left(1-2\gamma-2\sqrt{\frac{\gamma}{\log M}}\right).
\end{equation}
\label{thm_infor_lowerbound}
\end{thm}

With Lemma \ref{lem_angle_lowerbound}, \ref{lem_kl_upperbound} and Theorem \ref{thm_infor_lowerbound}
 we can prove Theorem \ref{thm_lower_bound} and Corollary \ref{cor_lower_bound} easily.
\begin{proof}[Proof of Theorem \ref{thm_lower_bound}]
We take $\mathcal F_0=\{P_{Y|X}^{(1)},\cdots,P_{Y|X}^{(M)}\}$
and $D(w,w')=\theta(w,w')$ in Theorem \ref{thm_infor_lowerbound}.
By Lemma \ref{lem_angle_lowerbound}, $D(w_i^*,w_j^*)\geq t$
and $\log M = \log|\mathcal W|\geq 0.0625d$ for $d\geq 2$.
In addition, Lemma \ref{lem_kl_upperbound} yields $\gamma\log M = 8C^2\cdot Tt^{2\alpha/(1-\alpha)}$.
Put $t=\kappa\cdot (d/T)^{(1-\alpha)/(2\alpha)}$ for some sufficiently small constant $\kappa>0$.
We then have $\gamma=O(1)$ and $\sqrt{\gamma/\log M}=O(1)$.
Consequently,
$$
\inf_{A\in\mathcal A_{d,T}^{\prob}}\sup_{P_{Y|X}\in\mathcal P_{\alpha,\mu_0}}\Pr\left(\theta(\widehat w,w^*)\geq \frac{t}{2}\right) = \Omega(1).
$$
Finally, applying Markov's inequality we obtain
$$
\inf_{A\in\mathcal A_{d,T}^{\prob}}\sup_{P_{Y|X}\in\mathcal P_{\alpha,\mu_0}}\mathbb E[\theta(\widehat w,w^*)]
\geq \frac{t}{2}\cdot\Omega(1) = \Omega\left((d/T)^{(1-\alpha)/2\alpha}\right).
$$
\end{proof}

\begin{proof}[Proof of Corollary \ref{cor_lower_bound}]
Suppose the density $g$ associated with $P_{X}$ is bounded from below with parameter $\gamma\in(0,1]$ (see Proposition \ref{prop_regression-excess-TNC}).
Let $\mathcal Q_{\alpha,\mu_0}$ denote the class of all conditional distributions $P_{Y|X}$ that satisfy Eq. (\ref{tnc-angle}) with parameter
$\mu_0=\mu/(2(1-\alpha)\gamma)$.
Note that $\mu_0$ does not depend on $d$ or $T$ when $\gamma$ is a constant.
By Proposition \ref{prop_regression-excess-TNC}, $\mathcal Q_{\alpha,\mu_0}\subseteq\mathcal P_{\alpha,\mu}$.
We then have
\begin{equation*}
\inf_{A\in\mathcal A_{d,T}^{\prob}}\sup_{P_{Y|X}\in\mathcal P_{\alpha,\mu}}{\mathbb E[\err(\widehat w)-\err(w^*)]}
\geq \inf_{A\in\mathcal A_{d,T}^{\prob}}\sup_{P_{Y|X}\in\mathcal Q_{\alpha,\mu_0}}{\mathbb E[\theta(\widehat w,w^*)^{1/(1-\alpha)}]}
= \Omega\left(\left(\frac{d}{T}\right)^{1/2\alpha}\right).
\end{equation*}
The first inequality is by Eq. (\ref{tnc-md}).
The second one is due to Jensen's inequality and Theorem \ref{thm_lower_bound}.
\end{proof}

\subsection{Proof of technical lemmas in Appendix \ref{appsec:proof_lowerbound}}\label{appsec:proof_lowerbound_technical}

In this section we prove Lemma \ref{lem_const_weight_coding_cor} used in the proof of Lemma \ref{lem_angle_lowerbound} in Appendix \ref{appsec:proof_lowerbound}.
We first cite a result from \cite{const-weight-coding} concerning the size of the largest separable set of constant-weight codes.
\begin{lem}[Theorem 7, \cite{const-weight-coding}]
Fix $d$ and $\omega\leq d$.
Let $X_{d,\omega}$ denote all $\omega$-weight binary codes of length $d$, that is,
$X_{d,\omega}=\{x\in\{0,1\}^d|\sum_{i=1}^d{x_i} = \omega\}$.
For any $\delta_H\leq d$, there exists a subset $S$ of $X_{d,\omega}$ such that
\begin{equation}
\Delta_H(x,x')\geq \delta_H,\quad\forall x,x'\in S, x\neq x'
\end{equation}
and 
\begin{equation}
\big|S\big| \geq \frac{\binom{d}{\omega}}{\sum_{i=0}^{\delta_H/2}{\binom{\omega}{i}\binom{d-\omega}{i}}}.
\end{equation}
Here $\Delta_H(\cdot,\cdot)$ denotes the Hamming distance.
\label{lem_const_weight_coding}
\end{lem}

We are now ready to prove Lemma \ref{lem_const_weight_coding_cor}.

\begin{proof}[Proof of Lemma \ref{lem_const_weight_coding_cor}]
We first comment that it is always possible to select $x_1,x_2\in\{0,1\}^d$ with $\|x_1\|_1=\|x_2\|_1 = \Delta_H(x_1,x_2)/2 = d/2$
whenever $d\geq 2$ and $d$ is even.
Consequently, for $2\leq d< 16$ the bound $\log |S|\geq 0.0625d$ always holds.
In the remainder of the proof we shall focus on the case when $d\geq 16$.

The following lower and upper bounds for binomial coefficient $\binom{n}{k}$ are well-known:
\begin{equation}
\left(\frac{n}{k}\right)^k\leq \binom{n}{k}\leq \left(\frac{en}{k}\right)^k.
\label{eq_binom_approx}
\end{equation}
Applying Eq. (\ref{eq_binom_approx}) and the lower bound in Lemma \ref{lem_const_weight_coding} we have
\begin{eqnarray}
\ln|S| &\geq& \frac{d}{2}\ln 2 - \ln\left(\frac{d}{32}\right) - \ln\binom{\omega}{\delta_H/2} - \ln\binom{d-\omega}{\delta_H/2}\nonumber\\
&\geq& \frac{d}{2}\ln 2 - \ln\left(\frac{d}{32}\right) - \frac{d}{16}\ln(16e)\nonumber\\
&=& d\cdot\left(\frac{1}{2}\ln 2 - \frac{1}{16}\ln(16e)\right) - \ln\left(\frac{d}{32}\right)\nonumber\\
&\geq& 0.11d - 0.02d = 0.09d.
\end{eqnarray}
The last inequality is due to the fact that $0.02x-\ln(x/32) \geq 0$ for all $x > 0$.
Therefore, $\log |S| = \ln |S|/\ln 2 \geq 0.13 d \geq 0.0625d$.
\end{proof}

\section{Proofs of some technical propositions}\label{appsec:technical_prop}

\begin{proof}[Proof of Proposition \ref{prop_pool_probe_reduction}]
Suppose we have a stream based algorithm $A\in\mathcal A_{d,T}^{\seq}$.
An equivalent query synthetic algorithm $B\in\mathcal A_{d,T}^{\prob}$ can be constructed based on $A$ as follows:
at iteration $t$ the algorithm $B$ repeatedly samples data points from $P_X$ until $A$ accepts a sample $x_t$;
it then picks $x_t$ and requests its label.
Clearly $A$ and $B$ are equivalent and they have the same worst-case expected excess risk.
\end{proof}

\begin{proof}[Proof of Proposition \ref{prop_regression-excess-TNC}]
Fix $w,w^*\in\mathbb R^d$ with $\|w\| = \|w^*\|=1$.
Denote $\Delta\subseteq \mathcal X$ as the set of data points on which $w$ and $w^*$ disagree; that is,
$\Delta = \{x\in \mathcal X: \sgn(w\cdot x)\neq \sgn(w^*\cdot x)\}$.
Suppose $\theta(w,w^*) = \theta$. 
We have the following:
\begin{equation}
\err(w) - \err(w^*) = \int_{\Delta}{2\bigg| \eta(x)-\frac{1}{2}\bigg|g(x)\ud x}
\geq 2\int_0^{\theta}{\mu_0\varphi^{\frac{\alpha}{1-\alpha}}\cdot \gamma\ud\varphi}
= 2(1-\alpha)\mu_0\gamma\cdot \theta^{\frac{1}{1-\alpha}}.
\end{equation}
\end{proof}

\end{document}